\pdfoutput=1
\documentclass[sigconf]{acmart}
\usepackage{framed}
\usepackage{balance}
\usepackage{epsfig}
\usepackage{color}
\usepackage{subfigure}
\usepackage{multirow,tabularx}
\usepackage{mathtools}
\usepackage{graphicx}
\usepackage{multirow}
\usepackage{bm}
\usepackage{csquotes}
\usepackage{array}
\usepackage[yyyymmdd,hhmmss]{datetime}
\usepackage{xcolor,colortbl}

\usepackage{pifont}
\usepackage[algo2e,titlenumbered,ruled]{algorithm2e} 
\usepackage{lipsum,environ,amsmath}
\usepackage{slashbox,booktabs}

\usepackage{pdflscape}

\newcounter{Lcount}
\newcommand{\numsquishlist}{
   \begin{list}{\arabic{Lcount}. }
    { \usecounter{Lcount}
 \setlength{\itemsep}{-.1ex}      \setlength{\parsep}{0ex}
      \setlength{\topsep}{0ex}       \setlength{\partopsep}{0ex}
      \setlength{\leftmargin}{1em} \setlength{\labelwidth}{1em}
      \setlength{\labelsep}{0.1em} } }
\newcommand{\numsquishend}{\end{list}}

\newcommand{\squishlist}{
   \begin{list}{$\bullet$}
    { \setlength{\itemsep}{-.1ex}      \setlength{\parsep}{0ex}
      \setlength{\topsep}{0ex}       \setlength{\partopsep}{0ex}
      \setlength{\leftmargin}{.8em} \setlength{\labelwidth}{1em}
      \setlength{\labelsep}{0.5em} } }
\newcommand{\squishend}{\end{list}}

\newcommand{\VLFMIP}{{\sc Variable-Lag Following Motifs Inference Problem }\xspace}%

\newenvironment{problem}[1][htb]
  {
   \begin{algorithm2e}[#1]%
  }{\end{algorithm2e}}

\AtBeginDocument{%
  \providecommand\BibTeX{{%
    \normalfont B\kern-0.5em{\scshape i\kern-0.25em b}\kern-0.8em\TeX}}}

\setcopyright{acmcopyright}
\copyrightyear{2023}
\acmYear{2018}
\acmDOI{XXXXXXX.XXXXXXX}

\acmConference[Conference acronym 'XX]{Make sure to enter the correct
  conference title from your rights confirmation emai}{June 03--05,
  2018}{Woodstock, NY}
\acmPrice{15.00}
\acmISBN{978-1-4503-XXXX-X/18/06}

\begin{document}

\title{ Framework for Variable-lag Motif Following Relation Inference In Time Series using Matrix Profile analysis
}


\author{Naaek Chinpattanakarn}
\orcid{0009-0005-0015-5736}
\affiliation{%
  \institution{King Mongkut's Institute of Technology Ladkrabang}
  \streetaddress{P.O. Box 1212}
  \city{Bangkok}
  \country{Thailand}
  \postcode{43017-6221}
}
\email{naaek.hugh@gmail.com}

\author{Chainarong Amornbunchornvej}
\orcid{0000-0003-3131-0370}
\affiliation{%
  \institution{National Electronics and Computer Technology Center}
  \streetaddress{112 Phahonyothin Road, Khlong Nueng, Khlong Luang District, Pathumthani 12120, Thailand. }
  \city{Pathumthani}
  \country{Thailand}}
\email{chainarong.amo@nectec.or.th}

\renewcommand{\shortauthors}{N. Chinpattanakarn and C. Amornbunchornvej}

\begin{abstract}
  Knowing who follows whom and what patterns they are following are crucial steps to understand collective behaviors (e.g. a group of human, a school of fish, or a stock market). Time series is one of resource that can be used to get insight regarding following relations. However, the concept of following patterns or motifs and the solution to find them in time series is not obvious. In this work, we formalize a concept of following motifs between two time series and present a framework to infer following patterns from data. The framework utilizes one of efficient and scalable methods to retrieve motifs from time series called matrix profile method. We compare our proposed framework with several baseline. The framework performs better than baseline in the simulation dataset. In the dataset of sounds from video, the framework is able to retrieve the following motifs within a pair to time series that two singers sink following each other. In the cryptocurrency dataset, the framework is capable of capturing the following motifs within a pair of time series from two digital currency. Our framework can be utilized in any field of time series.
\end{abstract}




\received{20 February 2007}
\received[revised]{12 March 2009}
\received[accepted]{5 June 2009}

\maketitle

\section{Introduction}

Followership is a process that one instance, a follower, interacts with other instance called a leader to achieve a collective goal~\cite{lapierre2014followership}, which is a counter part of leadership. Understand followership enables opportunities to get insight of how collective behaviors organized to serve the collective purpose, which can be a group of human~\cite{lapierre2014followership}, primates~\cite{strandburg2015shared}, or man-made system~\cite{FLICAtkdd}. Nevertheless, how can we infer followership from data? What are patterns that followers imitate their leader?

\begin{framed}
\noindent {\VLFMIP:} {Considering two time series, some patterns might appear in one time series, then, the same patterns might appear in another time series. This indicates that one time series follows the specific patterns of another time series.  {\bf Given two time series, the goal is to find a set of patterns (motifs) that occur in both time series with arbitrary time delays. }}
\end{framed}

In time series analysis, a \textbf{following relation} can be defined as a relation of two time series that one time series initiates patterns and another time series follows the similar patterns with some time delay~\cite{FLICAtkdd,mFLICASDM}. The following relation can be infer from time series regarding who follows whom using Dynamic Time Warping (DTW)~\cite{sakoe1978dynamic} by the framework proposed in~\cite{FLICAtkdd,mFLICASDM,mFLICASoftX}. However, there an issue when one time series is not completely follows another time series; a follower imitates only some patterns of its leader while the leader also imitates some patterns of its follower.  For instance, in Fig~\ref{fig:ExSimpleFoll}, the follower simply imitates a pattern of leader. In contrast, in Fig.~\ref{fig:ExMulMotifFoll}, there are four patterns or motifs that are followed by some time series. In this case, the follower imitates the motifs A and B of the leader, while the leader also follows the motifs C and D of the follower. This issue breaks the concept of the traditional following relation that assume the role of absolute leader and follower. 

\begin{figure}[ht!]
\includegraphics[width=1\columnwidth]{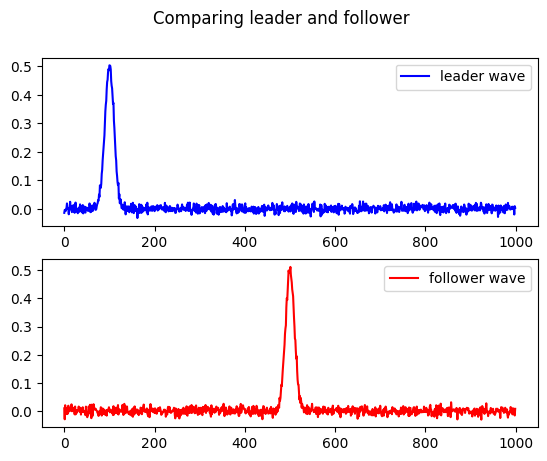}
\caption{Time series of following relation that the follower (red) follows the leader (blue) with a single motif.} 
\label{fig:ExSimpleFoll}
\end{figure}

\begin{figure}[ht!]
\includegraphics[width=1\columnwidth]{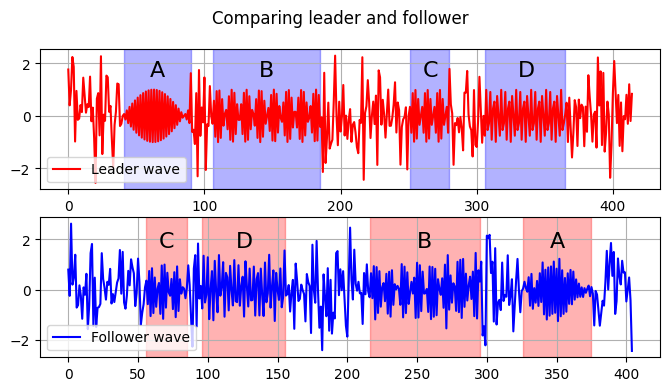}
\caption{Time series of following relation that the follower (blue) follows the leader (red) with multiple motifs (highlight intervals). In some motifs, the follower follows the leader (A and B motifs), while in some motifs, the leader follows the follower (C and D motifs.)} 
\label{fig:ExMulMotifFoll}
\end{figure}

To solve the issue, in this work, we propose the concept of the following-motif set. Specifically, instead of finding a following relation, we propose to find a set of motifs and roles of time series of each motif, which can tell us regarding 1) who initiates the motif, 2) who follows the motif, and 3) the time delay between initiation and imitation. We provide a problem formalization of inferring a following-motif set along with the framework using the matrix profile framework~\cite{zhu2016matrix} to infer it in time series. Our framework is capable of:
\squishlist
\item {\bf Inferring arbitrary-lag following-motif sets:} our methods can infer a set of motifs that occur in both time series with arbitrary time delays; 
\item {\bf Quantifying arbitrary-lag following-motifs:} our methods can report the similarity of following motifs between two time series for arbitrary delays. 
\squishend

We evaluate the performance of our framework using simulation datasets and manifests the use cases of our framework in two real-world datasets: the time series of sounds that two singers sings the same piece of lyrics, and the time series of cryptocurrency prices.  Our framework performed batter than baselines in most cases. The proposed framework can be utilized in any kind of time series. 

\section{Related work}
Most works of following relations are in the realm of time series that have real-world applications (e.g. flock of birds~\cite{andersson2008reporting}, behavior of pedestrians~\cite{kjargaard2013time}, a group of baboons~\cite{FLICAtkdd,mFLICASDM}, a school of fish~\cite{FLICAtkdd}, NASDAQ stock market~\cite{FLICAtkdd}). However, the definition of following relation can be different.

The work in~\cite{andersson2008reporting} defined the following pattern as a situation when one object is in the front of other individual. The work in~\cite{kjargaard2013time,FLICAtkdd,mFLICASDM,mFLICASoftX} defined a following pattern as a situation when one time series is similar to another time series with some time delay. The works~\cite{kjargaard2013time} in utilized the sliding window similar to the cross correlation method~\cite{conlon2009cross} while the works in~\cite{FLICAtkdd,mFLICASDM,mFLICASoftX} utilized the Dynamic time warping (DTW)~\cite{sakoe1978dynamic} to infer a following relations, which makes the framework more robust results since only utilizing sliding window cannot handle distortion between time series well. 
\\
However, the gap still remain since all these work cannot find multiple motifs from a pair of time series X and Y that some motifs might be initiated by X and be followed by Y but some might be initiated by Y and be followed by X. Moreover, the following-relation inference framework that deploys DTW or time-window sliding approaches cannot be scaled well. Especially, for DTW, we need a warping path from a dynamic programming matrix of DTW to find a following relation, which prohibits the utilizing of Keogh's bound~\cite{keogh2005exact} to reduce the computational time of DTW. To solve this problem of finding following motifs with who follows whom, we need a method to infer motifs efficiently.
\\
The motif discovery problem was introduced in the work~\cite{chiu2003probabilistic}, and since then there are several methods have been proposed to find motifs in time series~\cite{van2023variable,zhu2016matrix,liu2023enhanced,li2010approximate,linardi2018matrix}.  One of the well-known efficient methods to retrieve motifs from time series is to use the representation called \textit{Matrix Profile} introducing in~\cite{yeh2016matrix,zhu2016matrix}. The matrix profile inference frameworks were developed to discover motifs in time series in~\cite{yeh2016matrix,zhu2016matrix}. The frameworks were scalable to compute million data points to find exact motifs in reasonable time, which makes in one of the efficient framework to deploy for finding motifs. While matrix profile can be efficiently and scalably utilized to find many types of patterns in time series (e.g. motifs~\cite{zhu2016matrix}, discord~\cite{yeh2016matrix}, Shapelets~\cite{yeh2016matrix}, anomaly patterns~\cite{tafazoli2023matrix,lu2022matrix}), to the best of our knowledge, the matrix profile has not been utilized to solve the problem of finding following motifs yet.

In this work, we formalize \VLFMIP and propose a framework to find a following motifs in time series based on the matrix profile method in~\cite{yeh2016matrix}. In the next section, the problem formalization and related details are provided.

\section{Problem Statement and Algorithm}
\begin{definition}[Time series]
\label{def:TS}
A time series $U=(u_1,\dots,u_n)$  is a length-$n$ sequence of real values $(u_1,\dots,u_n)$ where $u_i\in\mathbb{R}$ is a value at time step $i$.
\end{definition}

\begin{definition}[Subsequence of time series]
\label{def:STS}
A subsequence $U_{t,m}$ of time series $U$ is a subsequence of $U$ starting at $t$ and ending at $t+m$:   $U_{t,m}=(u_t,\dots,u_m)$ where $1\leq t\leq m\leq n$ .
\end{definition}

\begin{definition}[Distance profile]
\label{def:DP}
Let $Q=(q_1,\dots,q_m)$ and $U=(u_1,\dots,u_n)$ be time series. A distance profile $D_Q(U)$ is a sequence vector of Euclidean distance of $Q$ over all possible subsequence $U_{t,m}$  of $U$ that have their length $m$. Formally, $D_Q(U)=(d_1,\dots,d_{n-m+1})$ where $d_t=\lVert Q-U_{t,m} \rVert_2$ 
\end{definition}

\begin{definition}[Similar join set]
Let $W$ and $U$ be time series and $\omega$ be a time window. Similar join set $J_{W,U}$ is a set of pairs of subsequences of $W$ and their set of $1$-nearest neighbor subsequences in $U$. 
Formally, suppose we have $d_{1nn}^U(W_{t,\omega})$ as a minimum distance between $W_{t,\omega}$ and its $1$-nearest neighbor subsequences in $U$.  

$J_{W,U}=\{ (W_{t,\omega},\hat{U}_t ) \}$ where

\begin{equation*}
    \hat{U}_t=argmin_{U_{\hat{t},\omega} \subseteq U}  (\lVert W_{t,\omega}-U_{\hat{t},\omega } \rVert_2 = d_{1nn}^U(W_{t,\omega})).
\end{equation*}

\end{definition}

\begin{definition}[Matrix profile~\cite{zhu2016matrix}]
A  matrix profile $P_{U,W}$ is a Euclidean distance between each subsequence $U_{t,\omega}$ and its $1$-nearest neighbor subsequences in $W$. Formally,
\begin{equation*}
    P_{W,U}=(d_{1nn}^U(W_{1,\omega}),\dots,d_{1nn}^U(W_{n-\omega+1,\omega})).
\end{equation*}

\end{definition}

\begin{definition}[variable-lag following motifs set]
\label{def:VLfm}
Let $W$ and $U$ be time series and $\omega$ be a time window. 
We call 
\begin{equation*}
    S_{W\preceq U}=\{ (W_{t,\omega},U_{t+\Delta_t,\omega})|W_{t,\omega}=U_{t+\Delta_t,\omega}\}
\end{equation*}

a set of variable-lag-following motifs of $U$ following $W$.
\end{definition}

According to Def.~\ref{def:VLfm}, we call variable-lag following motifs since each motif in W might be followed by some motif in U with arbitrary time delay and a different motif might be followed with a different time delay.

\begin{problem}
    \SetKwInOut{Input}{Input}
    \SetKwInOut{Output}{Output}
    \Input{Time series $W$ and $U$ as well as a time window $\omega$.}
    \Output{A set of following motifs of  $U$ following $W$: $S_{W\preceq U}$.}
    \caption{\VLFMIP}
\end{problem}

\setlength{\intextsep}{0pt}
\IncMargin{1em}
\begin{algorithm2e}
\caption{VLfollowingMotifsInferFunction}
\label{algo:VLfollowingMotifsInferFunction}
\SetKwInOut{Input}{input}\SetKwInOut{Output}{output}
\Input{ Time series $W$ and $U$ as well as a time window $\omega$. }
\Output{A set of following motifs of  $U$ following $W$: $S_{W\preceq U}$. }
\begin{small}
\SetAlgoLined
\nl Initial: $S_{W\preceq U} = \emptyset$\;
\nl \For{ each subsequence $W_{i,\omega},\subseteq W$}
{

\nl    Inferring  $\hat{U}_i=argmin_{U_{\hat{i},\omega } \subseteq U} (\lVert W_{i,\omega}-U_{\hat{i},\omega } \rVert_2 = d_{1nn}^U(W_{t,\omega}))$\;
\nl     \For{ each subsequence $U_{j,\omega},\subseteq \hat{U}_i$}
        {
\nl            Let $\Delta_i=j-i$\;
            \If{$\Delta_i\geq 0$ and $U_{j,\omega} = W_{i,\omega}$}
            {
\nl             $S_{W\preceq U}\leftarrow S_{W\preceq U} \cup \{(W_{i,\omega},U_{j,\omega})\} $\;
            }
        }
}
\nl Return $S_{W\preceq U}$\;
\end{small}
\end{algorithm2e}\DecMargin{1em}

\begin{proposition}
Algorithm~\ref{algo:VLfollowingMotifsInferFunction} is a solution of \VLFMIP.
\label{prop:siggamewin}
\end{proposition}
\begin{proof}
 (Forward direction) \\
Given time series $U,W$ and a time window $\omega$, for each $W_{i,\omega}$ in $W$, the algorithm finds its set of 1-nearest-neighbor subsequences $\hat{U}_i$ and it only includes a pair $(W_{i,\omega},U_{j,\omega})$ to $S_{W\preceq U}$ when the  motif $W_{i,\omega}$ appears the same time or before $U_{j,\omega}$, $j-i\geq 0$, and both subsequence must be equal. Since the algorithms iterates the same process for every subsequence $W_{i,\omega}$, the result $S_{W\preceq U}$ consistent with the definition $S_{W\preceq U}=\{ (W_{t,\omega},U_{t+\Delta_t,\omega})|W_{t,\omega}=U_{t+\Delta_t,\omega}\}$ in Def.~\ref{def:VLfm}.

(Backward direction) 
Given $S_{W\preceq U}=\{ (W_{ t,\omega},U_{t+\Delta_t,\omega})|W_{t,\omega}=U_{t+\Delta_t,\omega}\}$, assuming that there is $(W_{ \hat{t},\omega},U_{\hat{t}+\Delta_{\hat{t}},\omega}) \in S_{W\preceq U}$  but it is not a solution of the algorithm. 

The for loop in line 2 of Algorithm~\ref{algo:VLfollowingMotifsInferFunction} guarantees that the algorithm search for all possible subsequence of $W$, hence, $W_{\hat{t},\omega}$ must be analyzed by the algorithm. 

Suppose $\hat{U}_{\hat{i}}$ is a set of all 1-nearest-neighbor subsequences of $W_{\hat{t},\omega}$. Since $ U_{\hat{t}+\Delta_{\hat{t}},\omega} = W_{ \hat{t},\omega} $, $U_{\hat{t}+\Delta_{\hat{t}},\omega}$ must be a part of $\hat{U}_{\hat{i}}$ and being a part of the algorithm solution, which contradicts with the initial assumption.

Hence, Algorithm~\ref{algo:VLfollowingMotifsInferFunction} is a solution of \VLFMIP.

\end{proof}
\section{Methods}

Our proposed Following motif method measures leadership of leader time series by classifying time steps of following motif relation on each time series and subtracting following motif indices arrays of follower and leader time series to have a value. The greater the value is above zero, the more likely it is that the time series acts as a leader. 

\subsection{Matrix Profile inference}

For our method, the ability of Matrix Profile algorithm to find the most similar parts of two time series can be used to search for the same motif in both leading time series and following time series. 

To find the most similar parts of two inputted time series, Matrix Profile inference algorithm results in Matrix Profile (MP), which represents z-normalized Euclidean distance every subsequence and its (non-trivial) nearest neighbor within the same time series, and Matrix Profile Index (MPI), which represents time steps of the most nearest part of the inputs.

\setlength{\intextsep}{0pt}
\IncMargin{1em}
\begin{algorithm2e}
\caption{Matrix Profile inference algorithm}
\label{algo:matrixprofile}
\SetKwInOut{Input}{input}\SetKwInOut{Output}{output}
\Input{ Query Time series \( T_q \), Reference Time series \( T_{ref} \), subsequence length \( m \) }
\Output{Matrix Profile \( MP \), Matrix Profile Index \( MPI \)}
\begin{small}
\SetAlgoLined
\nl Initialize: \( MP \leftarrow \infty \), \( MPI \leftarrow 0 \)\;
\nl \For{ each sub-sequence \( T_{q(i,m)} \subseteq T_q \)}
{
\nl    Initialize: \( D \leftarrow \emptyset \)\;
\nl    \For{ each sub-sequence \( T_{ref(j,m)} \subseteq T_{ref} \)}
        {
\nl        \(
   D_j \leftarrow \sqrt{\sum_{k=0}^{m-1} \left( \frac{T_q[i+k] - \mu_{T_{q(i,m)}}}{\sigma_{T_{q(i,m)}}} - \frac{T_{ref}[j+k] - \mu_{T_{ref(j,m)}}}{\sigma_{T_{ref(j,m)}}} \right)^2}
   \)\;
        }
\nl    \For{ each element \( D_j \) in \( D \)}
        {
\nl            \If{\( D_j < MP[i] \) \textbf{and} \( i \neq j \)}
            {
\nl             \( MP[i] \leftarrow D_j \)\;
\nl             \( MPI[i] \leftarrow i \)\;
            }
        }
}
\nl Return \( MP, MPI \)\;
\end{small}
\end{algorithm2e}
\DecMargin{1em}

To illustrate how to perform matrix profile inference, the Matrix profile inference algorithm~\ref{algo:matrixprofile} is one example (STAMP) of many algorithms that can perform the task. Algorithm~\ref{algo:matrixprofile} does window sliding with reference time series by a window-size query, which cuts from the query time series, to calculate z-normalized euclidean distance in each slide. Thus, one sliding of the query results in a distance value for each time step of the reference time series. Excluding trivial matching time steps, Matrix profile and Matrix profile indices array are updated the lower distance value if the distance value is lower than their current element. After the matrix profile algorithm finishes every sliding window of cutting query and calculating distance for each query and reference time series.  Then, Matrix profile and the indices array, which has its size as the length of reference time series subtracted with window size, contain the shortest distances and the most similar query starting indices.  

The reason of z-normalization in matrix profile is to have the same size of time series to compare shapes of the two time series due to its main focus on searching similar shapes on its inputted time series.


\subsection{Following motif method}

Following motif method can provide set of leading motifs' indexes and following motifs' indexes, and determine which time series leads, from the two inputted time series: leader and follower time series. The method is composed of Matrix Profile algorithm, some set theory operations, percentile and average mean. 

\begin{algorithm2e}
\caption{Following Motif Method}
\label{algo:followingmotifmethod}
\SetKwInOut{Input}{input}\SetKwInOut{Output}{output}
\Input{Leader time series \( \text{leader\_ts} \), Follower time series \( \text{follower\_ts} \), Default value: Time window \( \text{wd} = 300 \), Threshold of Motif and Discord \( \text{thresh} = 0.01 \)}
\Output{Leading Result,Leading Result Value,and Index Arrays of Following Motif Relation for Leader and Follower time series}
\begin{small}
\SetAlgoLined
\nl Initialize: \( index_{motif} = \{\emptyset,\emptyset\}\)\;
\nl \( MP_{leader},MPI_{leader} \leftarrow \text{MatrixProfile}(follow_{ts}, wd, lead_{ts})\)\;
\nl \( MP_{follower},MPI_{follower} \leftarrow \text{MatrixProfile}(lead_{ts}, wd, follow_{ts})\)\;
\nl \( MP_{array} = \{MP_{follower},MP_{leader}\}\)\;

\nl \For {each sub-sequence \( MP_{i} \subseteq MP_{array} \)}
{
\nl     \(motif_{percentile} = \ (50-thresh)\text{-th percentile value of } MP_{i} \)\;
\nl     \(index_{motif}[i] \leftarrow \{ i \mid MP_{i} \textless motif_{percentile}, i \in \{0, \ldots, \text{len}(MP_{i})-1\} \)\;
}

\nl \( index_{motif}[1] = index_{motif}[1][1:\text{len}(index_{motif}[0])]\)\; 
\nl \( index_{difference} = index_{motif}[0] - index_{motif}[1]\)\; 
\nl \( result_{leadvalue} \leftarrow \frac{1}{\text{len}(index_{difference})} \sum_{i=1}^{\text{len}(index_{difference})} (index_{difference}[i])\)\;  
\nl \(result_{lead} = (result_{leadvalue} > 0)\)\;  

\nl Return \( result_{lead},result_{lead value}, index_{motif} \); 
\end{small}
\end{algorithm2e}\DecMargin{1em}

In Algorithm~\ref{algo:followingmotifmethod}, Matrix Profile's algorithm results the array of the shortest distance between a part, whose size is time window, of reference time series and query time series. The lower value of the shortest distance array at some indices, the more likely the motifs exist at the indices. Thus, Matrix Profile of leader time series, which has leader time series is reference time series, provides low values at the indices that the time series contains motifs.The similar case occurs for the follower time series. 

In line 2-3 of Algorithm~\ref{algo:followingmotifmethod}, \(MP_{leader},MP_{follower}\) Matrix Profiles of leader and follower time series are the arrays that can be extracted for the areas of motifs on both leader and follower time series. The algorithm version of Matrix Profile's algorithm is STAMP algorithm due to its performance in our problem. 

Using Matrix Profile algorithm for having motifs found area on the two time series in Algorithm~\ref{algo:VLfollowingMotifsInferFunction} and Algorithm~\ref{algo:followingmotifmethod} are similar but the difference is that Algorithm~\ref{algo:VLfollowingMotifsInferFunction} line 6 uses the loop for collecting the most similar parts, which the follower's part follows the leader's part, after every distance profile calculation. Meanwhile, Algorithm~\ref{algo:followingmotifmethod} uses the two Matrix Profiles with percentile threshold. 

The reason of using the threshold in Algorithm~\ref{algo:followingmotifmethod} is to prevent mismatching, excluding trivial match, of the most similar parts of the two time series. In many cases, many leading motifs on leader time series are similar enough and result in following the wrong motifs on follower time series; the following motifs mismatch the another similar leading motif, which is not its own leading motif. 

Moreover, similarity measurement of the two parts of the time series cannot be $U_{j,\omega} = W_{i,\omega}$ in the noisy and real-world setting because the following lag is arbitrary of every time steps, which provides chance of different length of the two parts on the two time series and the varied shape of the following motif. Those issues are also a reason of the mismatching. If we uses similarity measurement for different length of two Time Series such as Dynamic Time Warping, and Longest Common Subsequence instead. It will increase computational complexity due to their characteristic of dynamic programming. 

In line 4, we pack \(MP_{follower}\) and \(MP_{leader}\) in an array for the loop purpose. Thus, the first and second of the array's index are associated with follower and leader respectively. 

To extract time steps of the leader and follower where motif exists, in between line 6 and 7, we determine the thresholds of motif with the assumption that the shortest distance, where motifs are found on the two following time series, is from 0th to (50-thresh)th percentile value of Matrix Profile. The parameter thresh is used for tuning the threshold to cover all of the motif area. Subsequently, we use the thresholds to create new indices arrays of motifs. 

In line 8, if motifs on leader time series leads the motifs' pattern on follower time series, the length of motifs on the leader must be equaled to the length of the motifs on the follower. 

To determine leader and follower time series by the method, the method element-wise subtracts \(index_{motif}\) of the follow by \(index_{motif}\) of the leader as in line 9 because the indexes of the motifs on the follower are behind the leader's indexes. The more subtracted elements are positive, The more likely the inputted leader time series leads another time series.  

In line 10, result of leading in value is assigned by average mean of the index difference array. The leading value is sum of the difference array, that have its center at zero to separate two classes(leader and follower), and scale down the value by dividing length value of the difference array. The further leading value to zero, the more confident the inputted leader is the real leader time series. If the leading value is more than 0, the inputted leader is real leader. 

Default input parameters - such as time window, gap of motif (percentage of inputted time series' length) - are the best set of parameters for the synthesis and real data set of our experiment. 

In the aspect of computational cost, the proposed algorithm has time complexity as $O(n^2 log(n))$ since we use STAMP~\cite{yeh2016matrix}  as the main computation part where $n$ is the length of time series. Moreover, note that we could also use SCRIMP++~\cite{zhu2018matrix} which is only $O(n^2)$. Both STAMP and SCRIMP++ allow computation both as anytime and contract algorithms~\cite{lopez2014optimal}. In practice this implies that datasets with up to say ten million datapoints can easily be analyzed in realtime interactive sessions.

\section{Experiments}

The methods were evaluated with three types of synthetic time series simulation (1,000 time series for each) and two real time series of music and Cryptocurrency domains. The synthetic and real world data sets were tested with two tasks: Determining which time series leads or follows, and predicting time steps of following motifs task.

\subsection{Experimental setup}

To determine which time series leads or follows, we compared performance of the following motif method with three existing methods, which have their parameter leader and follower time series: FLICA~\cite{FLICAtkdd,mFLICASoftX}, Maximum Cross Correlation, and Variable-Lag Transfer Entropy (VL-Transfer Entropy)~\cite{10.1145/3441452}. 

For the experiment of those methods with every pair of synthetic and real world time series, suppose A is leading time series and B is a following time series as a ground truth (GT). 1), if the predicted result (PR) agrees with GT, then the pair of time series is counted as a true positive case. 2) if A does not follow B, then it is counted as a false negative case. In contrast, 3) if PR states that B follows A, which contradicts GT, then it is counted as a false positive case. Lastly, 4) if B does not follow A, then it is counted as a true negative case.

To predict the following motif time steps, following motif method is the only method that results in a set of following motifs time steps. If the predicted following motif time step matches following motif ground truth, one true positive is counted. Otherwise, if the predicted following motif time step does not match following motif ground truth, it is counted as a false negative. If the predicted non-following motif time step matches non-following motif ground truth, one true negative is counted. Otherwise, if the predicted non-following motif time step matches the following motif time step, it is counted as a false positive.

\subsection{Synthetic time series simulation}

We have three synthetic time series simulations:
Fixed length single motif simulation, Continuous fixed length multiple motif simulation, Noncontinuous fixed length multiple motif simulation. Each simulation has 1,000 pairs of Leader and Follower Time Series.  generated from random seed 0 to 999. The patterns in each simulation are created by Sine function:

\begin{figure}
    \centering
    \includegraphics[width=1\linewidth]{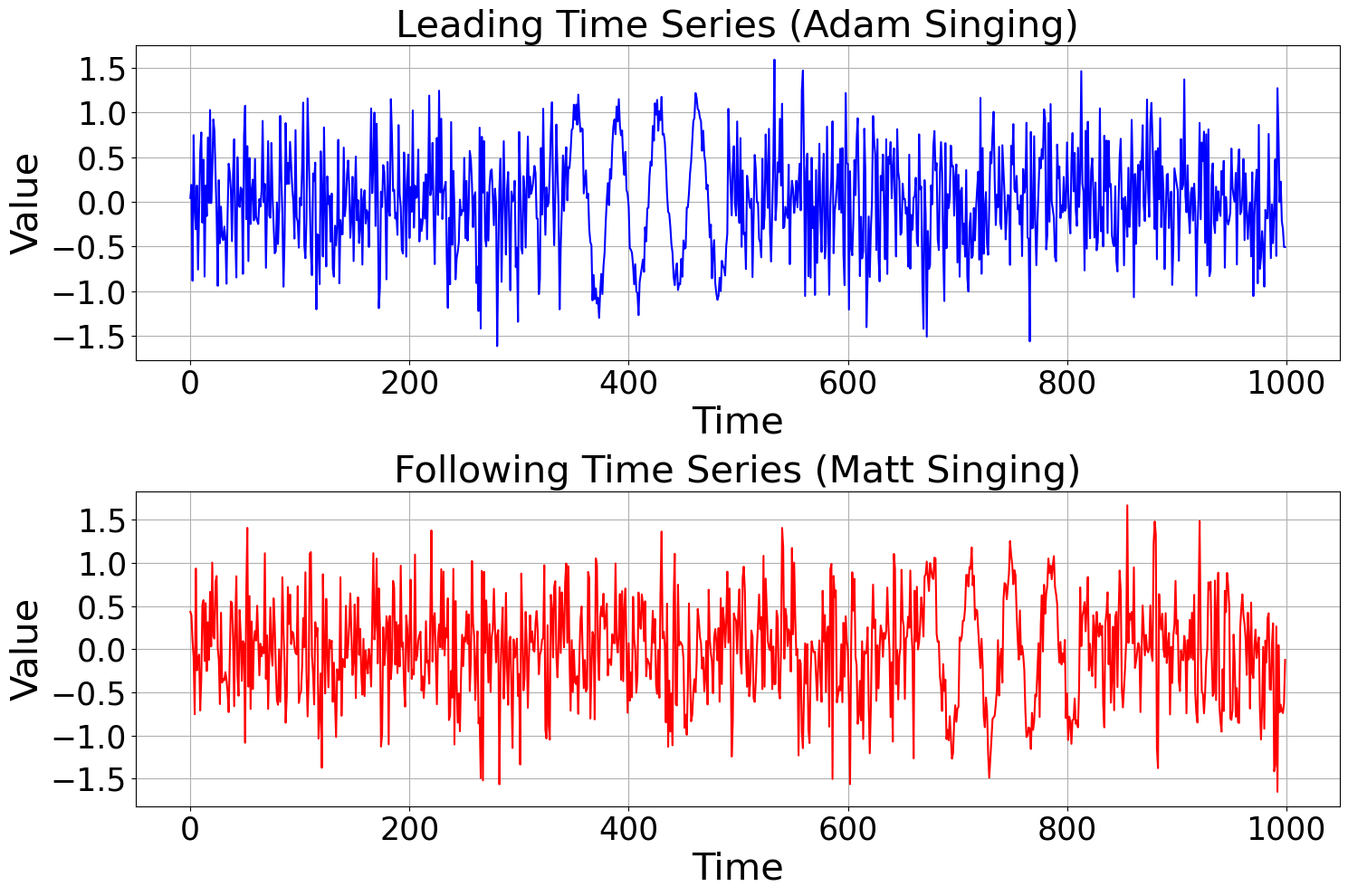}
    \caption{Fixed length single motif simulation with random seed 0}
    \label{fig:dataset1}
\end{figure}

\begin{figure}
    \centering
    \includegraphics[width=1\linewidth]{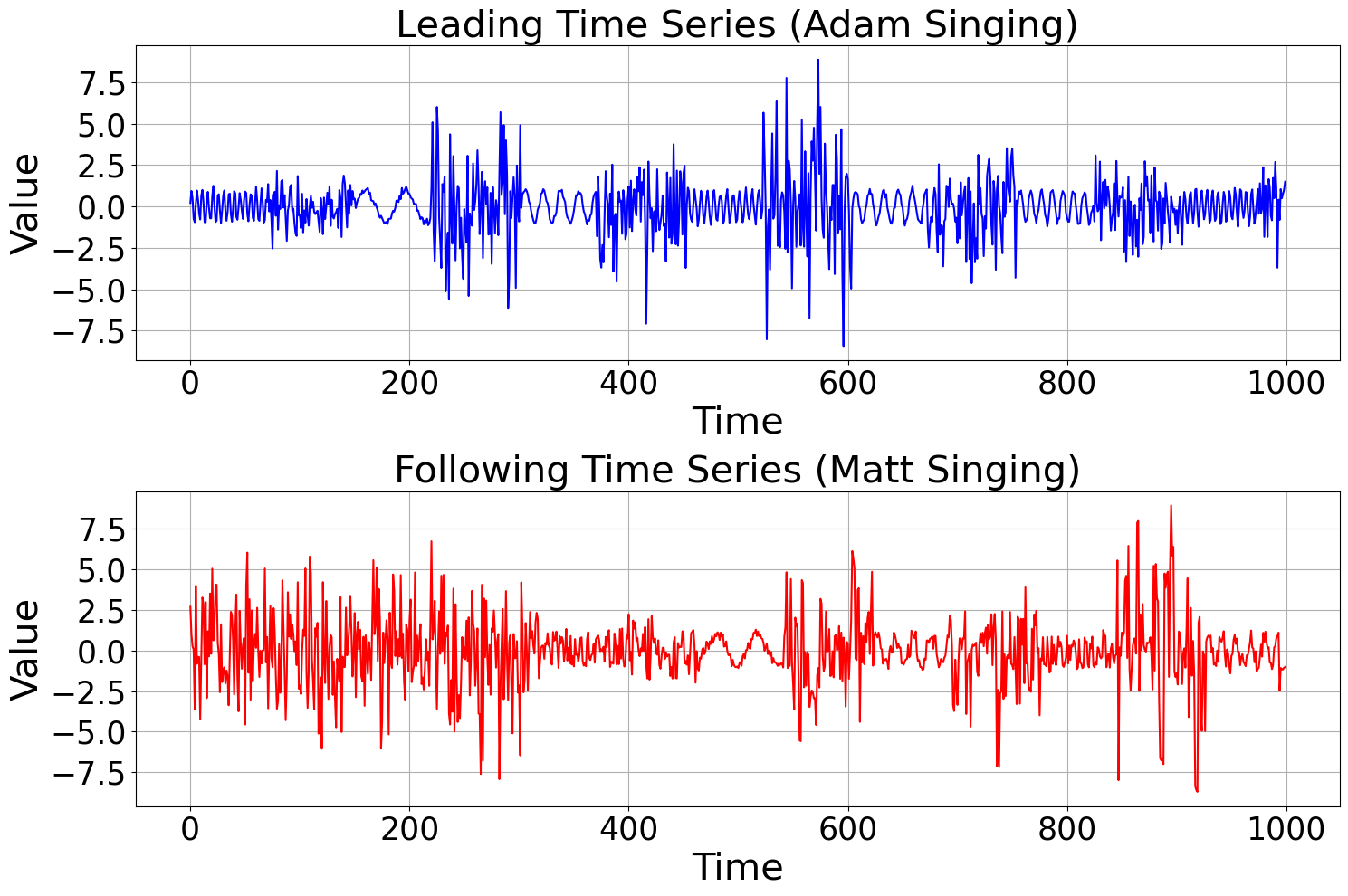}
    \caption{Continuous fixed length multiple following motifs simulation with random seed 1}
    \label{fig:dataset2}
\end{figure}

\begin{figure}
    \centering
    \includegraphics[width=1\linewidth]{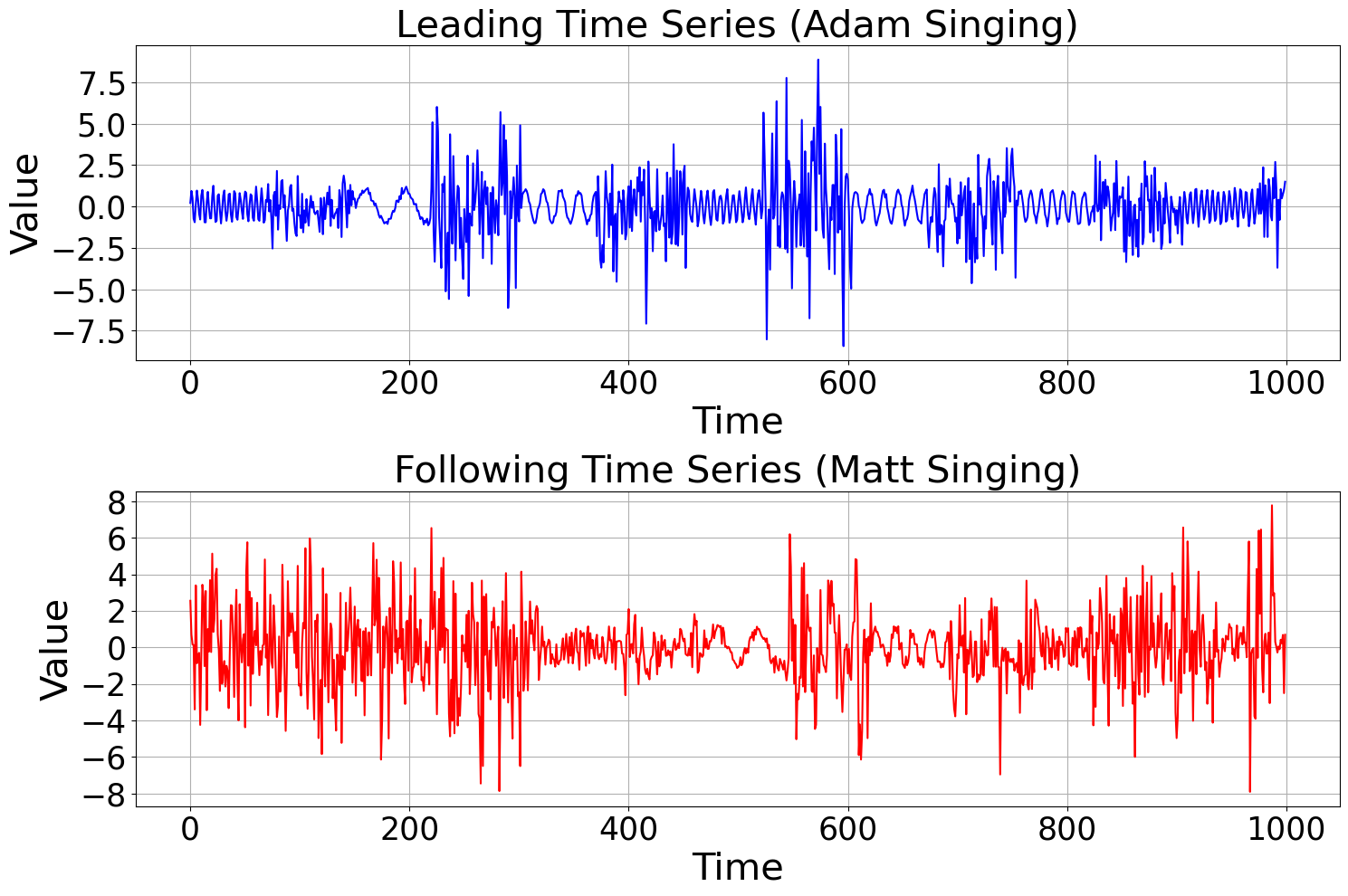}
    \caption{Noncontinuous fixed length multiple following motifs simulation with random seed 1}
    \label{fig:dataset3}
\end{figure}


\begin{equation}
\begin{aligned}
\left\{ n \cdot \sin\left(\frac{x_1}{x_2}\right) \right\} \quad & \text{for } n \in \{0, 1, 2, \ldots, \text{motif length}\}, \\
\text{where } & \underbrace{x_1}_{\text{in } [0.2, 2.0]}, \underbrace{x_2}_{\text{in } [3.00, 5.00]}.
\end{aligned}
\label{eq:MotifEquation}
\end{equation}

\squishlist
\item {\bf Single Motif:} For each pair of time series, we generate one motif that appears in the leader time series first, then, the same motif appears in the follower time series later with some time delay. 
\item {\bf Continuous Following:} The continuous following is following motif patterns that does not break from noisy interruption, of the follower time series. The motif length is random from
\[\lceil \text{length}(\text{time series}) \times \text{percent}(2.5\% \text{ to } 10\%) \rceil\]  and the length of non-motif is 
\[ \lceil \text{length}(\text{motif}) \times \text{percent}(120\%) \rceil \].
The maximum motif patterns that leader can have is 10. 
\item {\bf Noncontinuous Following:} The noncontinuous following data set is similar to continuous following data set except adding more noisy part to interrupt the half way of time steps in the following part:
\[
\begin{pmatrix}
\text{First half following part} & \text{Interrupting part} & \text{Last half following part}
\end{pmatrix}
\]
The size of new noisy part is 10 to 15 percent of the entire time series and it is generated by normal distribution random with mean = 0 and standard deviation = 0.5. The size of the follower is same as continuous following data set. Because after adding the new noisy part, we trim the part that beyond its original size. 
\item {\bf Mixed Types:} The mixed type data set (3,000 pairs of time series) is composed of the Single Motif (1,000 pairs), Continuous Following (1,000 pairs), and Noncontinuous Following data set (1,000 pairs). 
\squishend

For every type of data set, the motif pattern is generated by equation~\ref{eq:MotifEquation}. On follower time series, the following starting time steps starts at \[
\lceil \text{length}(\text{leader}) \times \text{percent}(20\% \text{ to } 35\%) \rceil
\] 
The non-following motif in both time series are filled with random noise, generated by normal distribution random with mean = 0 and standard deviation = 0.5. The lag of follower is varied from the leader in range of 0 to 3. The follower is added normal distribution random with mean = 0 and standard deviation = 0.1 twice. Every random in each pair is determine by random seed from 0 to 999.


\subsection{Real world dataset}

We collect real world time series from music and Cryptocurrency domains: a part of the song Lost Stars from the 2014 finale of The Voice, performed by Adam Levine and Matt McAndrew, from second 12 to 45, and Cryptocurrency opening price trends of BTC/USDT and ETH/USDT in June 10,2023 in 30 minutes for 24 hours.  

\subsubsection{Vocal extracted sounds of Lost Stars by Adam Levine and Matt McAndrew}

The vocal sounds of Adam and Matts taking turns singing, as in Figure~\ref{fig:loststar}, are two time series that each contain the singing of "Please, don't see. Just a boy caught up in dream and fantasies", and "Take me hand Let's see where we wake up tomorrow" for Adam and the singing of "Please, see me. Reaching out for someone I can't see", and "Best laid plans, Sometimes are just a one night stand" for Matt. We collect the sound data from \url{https://www.youtube.com/watch?v=YFsGH5iToys} in .wav format, then we extracted vocals from the song, trimmed the vocal data from 0:12 to 0:45, and separated vocals of Adam and Matt by Audacity. The separated vocals are loaded by Python Librosa with sampling rate = 11025, then the vocals are sampled down by average groups of data points, with each group containing an equal fraction (5\%) of the data points.

\begin{figure}
    \centering
    \includegraphics[width=1\linewidth]{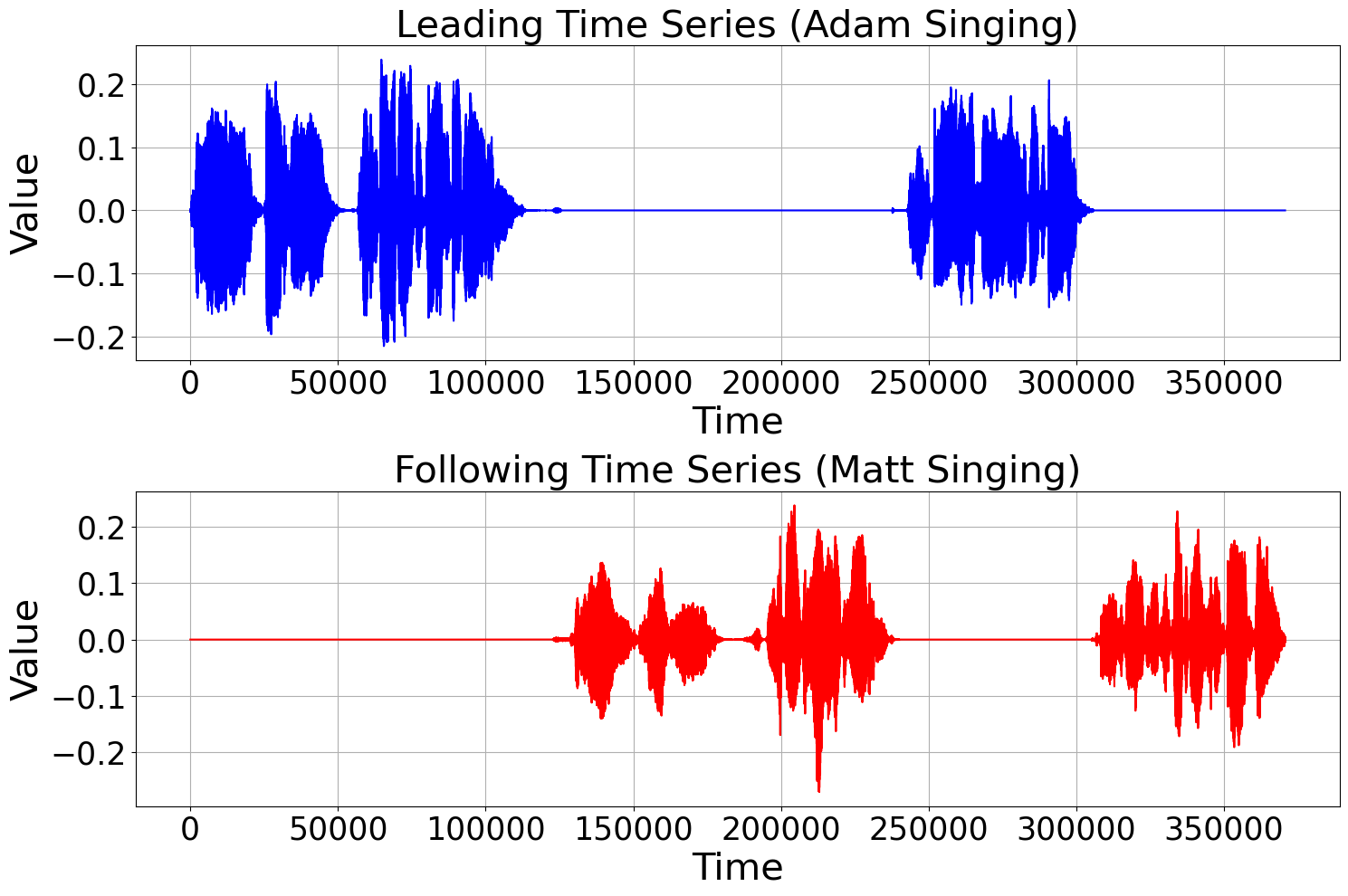}
    \caption{Real World Time Series1: Lost Star by Adam and Matt}
    \label{fig:loststar}
\end{figure}


\subsubsection{Cryptocurrency Opening Price of Bitcoin/USDT and Ether/USDT Time Series}

The opening-price data of the two Cryptocurrency are collected by 30 minutes of interval in June,10 2023 (1300 time-steps). In Figure~\ref{fig:bitcoin}, normalization of the two price Time Series is to compare the shape and lag of following relation. Most of the time steps of Bitcoin's price leads the shape trend of Ether's price. We collect the the data via Binance's API with Python library python-binance. 

\begin{figure}
    \centering
    \includegraphics[width=1\linewidth]{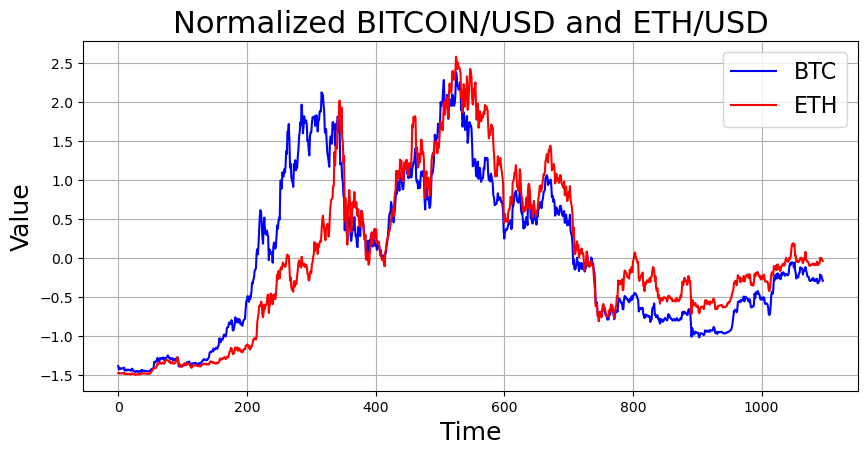}
    \caption{Real World Time Series2: Normalized BTC/USD and ETH/USD to compare leading and following interval}
    \label{fig:bitcoin}
\end{figure}

\section{Results}


\subsection{Synthetic time series simulation}

\subsubsection{Performance of methods for determining which time series leads or follows}

Table~\ref{tab:intervalInfTask} reports comparison of Following Motif Method, FLICA, Cross Correlation, and Varied lag-Transfer Entropy to result which time series in each pair of Following Relation time series is leader. Each data 
set contains 1,000 pairs of leader and follower and mixed type data set is merged of the other data sets. For VL-Transfer Entropy, we assume that the following motif patterns is occurred by the cause of leader time Series on the follower. Thus, the leader causes the follower to follows the leader's motifs.

Table~\ref{tab:intervalInfTask} illustrates the best performance of Following Motif Method in precision, recall, F1, and accuracy at the data sets of continuous following, non-continuous following and mixed types. For single motif data set, Following Motif Method beats Cross Correlation in Recall, F1, and Accuracy and have slightly lower Precision of 0.006. This is because finding the lag that produces the maximum value of the dot product is slightly more effective than finding similarity of time series' shape by normalized Gaussian distance function of matrix profile inside Following Relation Method, Even though thick noise is added and disturbs comparing shape of the following motif relation.    

Reason of similar values of metrics at every data set for FLICA is the similar number of True Positive and True Negative and of False Positive and False Negative. The reason of similar values also exists for the result of Cross Correlation at Continuous and Non-continuous following data sets. 

\begin{table*}[htbp]
\begin{tiny}
\centering
\resizebox{\textwidth}{!}{%
\begin{tabular}{|c|c|c|c|c|c|}
\hline
Dataset & \multicolumn{1}{c|}{Result} & Following Motif Method & FLICA & Cross Correlation & VL-Transfer Entropy \\
\hline
\multirow{4}{*}{Single Motif} & Precision & 0.995 & 0.693 & \textbf{1.000} & 0.774 \\
& Recall & \textbf{0.997} & 0.693 & 0.986 & 0.597 \\
& F1 & \textbf{0.996} & 0.693 & 0.993 & 0.674 \\
& Accuracy & \textbf{0.996} & 0.693 & 0.993 & 0.712 \\
\hline
\multirow{4}{*}{Continuous Following} & Precision & \textbf{1.000} & 0.732 & 0.872 & 0.241 \\
& Recall & \textbf{1.000} & 0.732 & 0.872 & 0.111 \\
& F1 & \textbf{1.000} & 0.732 & 0.872 & 0.152 \\
& Accuracy & \textbf{1.000} & 0.732 & 0.872 & 0.381 \\
\hline
\multirow{4}{*}{Noncontinuous Following} & Precision & \textbf{0.984} & 0.663 & 0.684 & 0.203 \\
& Recall & \textbf{0.983} & 0.663 & 0.684 & 0.085 \\
& F1 & \textbf{0.983} & 0.663 & 0.684 & 0.120 \\
& Accuracy & \textbf{0.984} & 0.663 & 0.684 & 0.376 \\
\hline
\multirow{4}{*}{Mixed Types} & Precision & \textbf{0.993} & 0.696 & 0.851 & 0.406 \\
& Recall & \textbf{0.993} & 0.696 & 0.847 & 0.264 \\
& F1 & \textbf{0.993} & 0.696 & 0.849 & 0.315 \\
& Accuracy & \textbf{0.993} & 0.696 & 0.850 & 0.489 \\
\hline
\end{tabular}%
}
\caption{Performance of methods for determining which time series leads or follows. The bold text represents the best performer.}
\label{tab:intervalInfTask}
\end{tiny}
\end{table*}

\subsubsection{Performance of Following Motif Method for the finding of time steps of following motifs task}

Finding time steps of following relation on both leader and follower Time Series "Which time steps is associated with the following motif relation?" can only be calculated by only Following Motif Method. 

Table~\ref{tab:intervalInTask1} reports result of the finding Following Relation time steps at data sets of single motif, continuous following motif, and mixed type. The significantly highest all metrics of single motif data set shows better performance to multiple following motif. The obviously higher accuracy than F1-score for every data set indicates that the Following Motif Method is effective to detect non-related the following motif relation time steps from the entire time series.



\begin{table}[ht]
  \begin{minipage}[t]{0.5\textwidth}
    \centering
    \begin{tabular}{|c|c|c|c|c|}
      \hline
      Dataset & Precision & Recall &  F1-score & Accuracy\\ \hline
      Single  & 0.703     & 0.802  & 0.703     & 0.915 \\ \hline
      Cont.   & 0.553     & 0.708  & 0.621     & 0.705\\ \hline
      Mixed   & 0.591     & 0.737  & 0.656     & 0.819\\ \hline
    \end{tabular}
    \caption{Performance of Following Motif Method for inferring time steps of following motifs task.}
    \label{tab:intervalInTask1}
  \end{minipage}%
  \hfill
  \begin{minipage}[t]{0.5\textwidth}
    \centering
    \begin{tabular}{|c|c|c|c|c|c|}
      \hline
      Noise & LeadVol & Precision & Recall &  F1-score & Accuracy\\ \hline
      0.000 & 0.168 & 0.916     & 0.862  & 0.889     & 0.862 \\ \hline
      0.001 & 0.215 & 0.960 & 0.891  & 0.924     & 0.906\\ \hline
      0.005 & \textbf{0.235} & \textbf{0.962} & \textbf{0.898} & \textbf{0.929} & \textbf{0.911}\\ \hline
    \end{tabular}
    \caption{Following Motif Method's performance for determining which time series leads and inferring following motif time steps with real vocal time series, added different degrees of noise.}
    \label{tab:intervalInTask2}
  \end{minipage}
\end{table}


\subsection{Real datasets}

The following motif method's result of determining which time series leads of real world data set are correct for both Lost Star singing and Bitcoin and Ether price time series. 

\subsubsection{Lost Stars Singing by Adam Levine and Matt McAndrew}

Table~\ref{tab:intervalInTask2} reports performance of following motif method in both tasks with different degrees of normal distribution noise of standard deviation 0.00, 0.001, and 0.005. The table shows the metrics for every degrees of noise above 0.86 for predicting following motif time steps. Figure ~\ref{fig:FRloststar1} illustrates green(predicted), red(ground truth) and yellow(the predicted overlapped ground truth) highlighted time steps as following motif relation interval on Leader (Adam singing) and Follower (Matt singing). The table and figures are resulted from Following Motif Method with parameters (window size 700 and percentile gap 10) with the three degrees of noise.

The LeadVol in the table~\ref{tab:intervalInTask2} are the leadership value of the inputted leader time series. The greater LeadVol from zero, the inputted leader likely is real leader. Every degree of noise shows significantly high LeadVol which determine correct leader to the ground truth. 

The method can find the following pattern in rhythm of singing even the leader and follower sing different lyrics. The following pattern is almost entirely detected and covered inside the highlighted by the method. The method determine Adam singing leading Matt singing.  

The ground truth following motif interval of leader are [0,5700] and [12100,15100]. For the follower, the interval are [6400,12000] [15300,18533].

\begin{figure}
    \centering
    \includegraphics[width=1\linewidth]{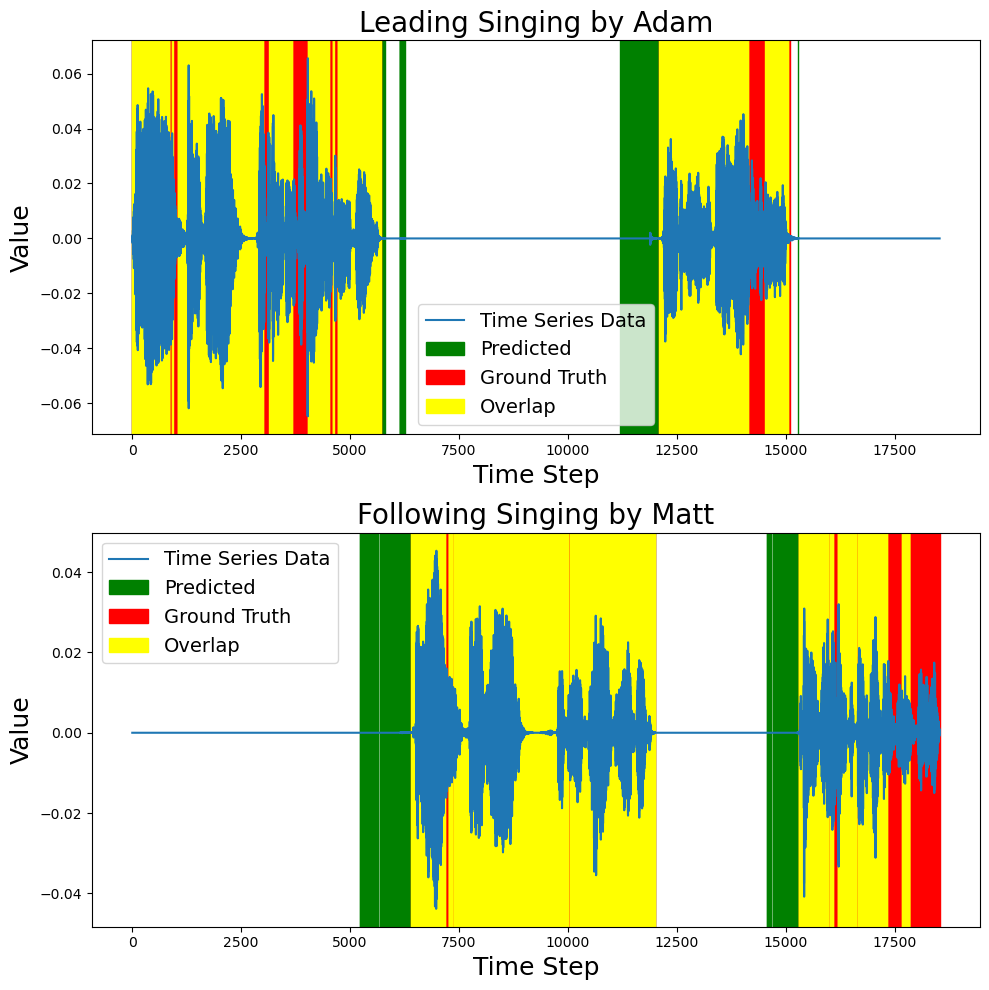}
    \caption{Following Relation Method indicates following ordered points on leader and follower Time Series of Lost Star Singing with Resolution (Gaussian noise Std. = 0.000)}
    \label{fig:FRloststar1}
\end{figure}



\subsubsection{Cryptocurrency: Bitcoin and Ether to USDT}

The following method correctly determine Bitcoin as leader and results in LeadVol of 0.0198. 

Figure~\ref{fig:FRbitcoin1} illustrate highlighted time steps as following motif relation interval on Leader (Bitcoin Opening Price Signal) and Follower (Ether Opening Price Signal), resulted from Following Motif Method with parameters (window size 150 and percentile gap 5). The time steps that Bitcoin is leading the Ether motifs,that can be detected in window size 150, are highlighted.  

In Figure~\ref{fig:bitcoin}, although the following lags of those two cryptocurrency are varied and small, the following relation method still can effectively detect the time steps of Bitcoin leading Ether's motifs. We concluded the performance by comparing Figure~\ref{fig:bitcoin} (groundtruth) with Figure~\ref{fig:FRbitcoin1} (result). The method determine Bitcoin leading Ether.   

\begin{figure}
    \centering
    \includegraphics[width=1\linewidth]{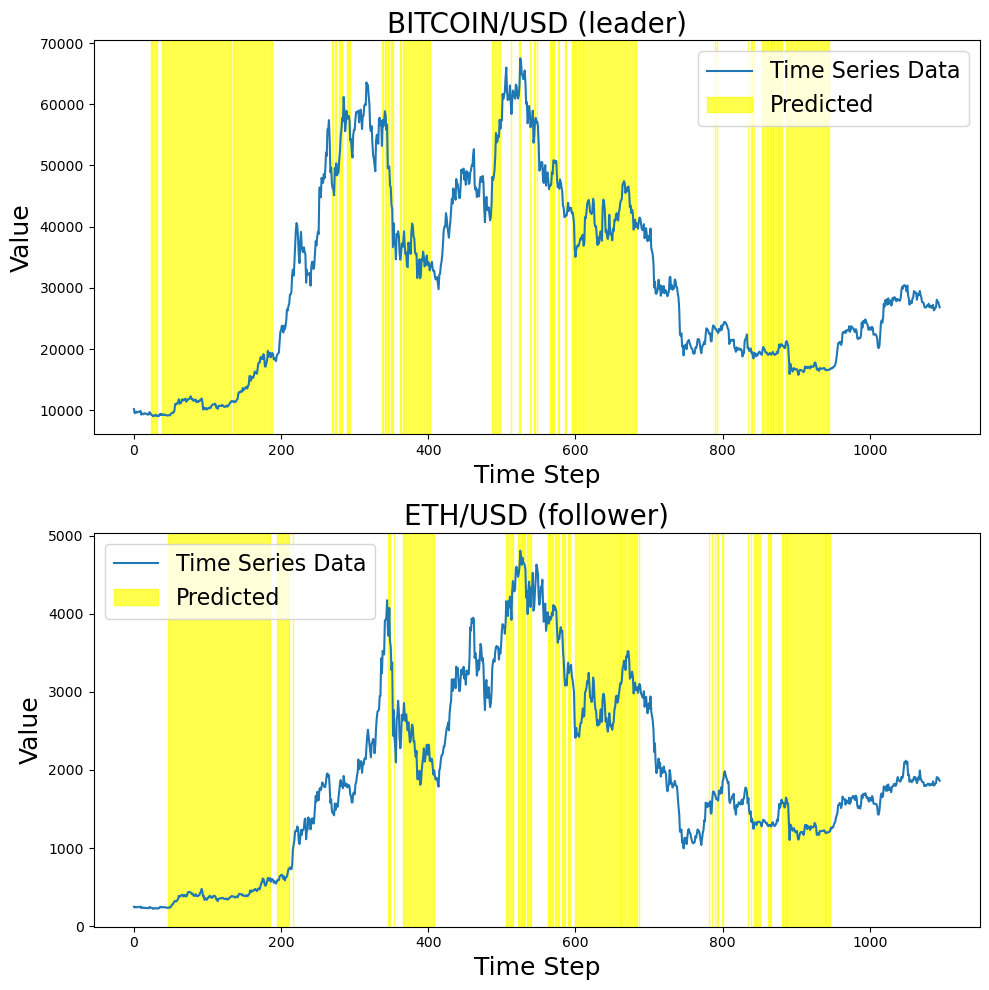}
    \caption{Yellow highlighted interval as following motif relation interval in Leading and Following Motif Time Series (Price of Bitcoin and Ether)}
    \label{fig:FRbitcoin1}
\end{figure}



\section{Conclusion}
In this work, we formalized a concept of following motifs between two time series and presented a framework to infer following patterns from data. The framework utilized one of efficient and scalable methods to retrieve motifs from time series called matrix profile method. We compared our proposed framework with several baselines. The framework performed better than baselines in the simulation datasets. In the dataset of sounds recording, the framework was able to retrieve the following motifs within a pair of time series that two singers sang following each other. In the cryptocurrency dataset, the framework was capable of capturing the following motifs within a pair of time series from two digital currencies, which implied that the values of one currency followed the values of another currency patterns.  Our framework can be utilized in any field of time series to get insight regarding following patterns between time series. 

The code and datasets can be found at \href{https://github.com/hughnaaek/Following-Motif-Relation}{https://github.com/hughnaaek/Following-Motif-Relation}.

\balance
\bibliographystyle{ACM-Reference-Format}

\appendix

\section{Case Study: Simple Alignment of Following Relation Interval in Synthetic and Real Dataset}

To simply align the following motif relation time steps of leader and follower time series, the Following Relation Method has index arrays of following relation time steps of leader and follower. Assuming that the follower time series follows similar size of the leader's motif and the method produces the following relation time steps of both Time Series, we can map the index arrays as follow relation alignment of leader and follower as Figure~\ref{fig:FRdataset2}. 

In the case that Following Relation Method detect some non-motif as False Positive, the simple alignment still more effectively aligns the time steps as Figure~\ref{fig:FRdataset2}, compared to DTW aligns following motif pairs of time series in Figure~\ref{fig:DTWdataset2}, where the alignment is wrong for entire time steps then it causes wrong alignment throughout the two Time Series. That is the reason why FLICA and VL-Transfer Entropy, which use DTW as a core of algorithms, are ineffective determine which time series in pair is a leader because of mixing of motif and non-motif inside the Time Series. 

Figure~\ref{fig:FRbitcoin} illustrates following motif relation alignment with Real-world Datasets of Cryptocurrency price trends of Bitcoin and Ether.

\begin{figure}
    \centering
    \includegraphics[width=0.9\linewidth]{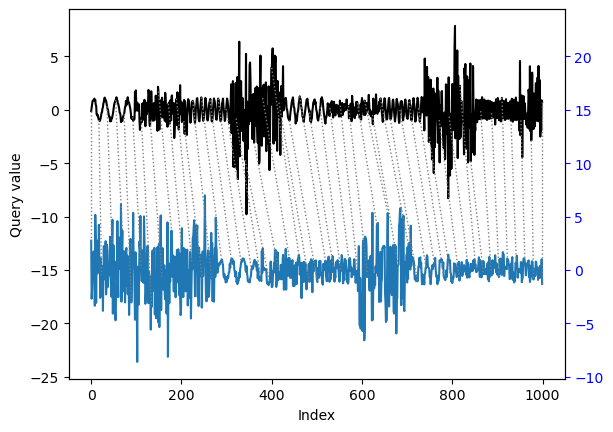}
    \caption{DTW aligns leader and follower Time Series of Dataset2 randomseed 883}
    \label{fig:DTWdataset2}
\end{figure}

\begin{figure}
    \centering
    \includegraphics[width=1\linewidth]{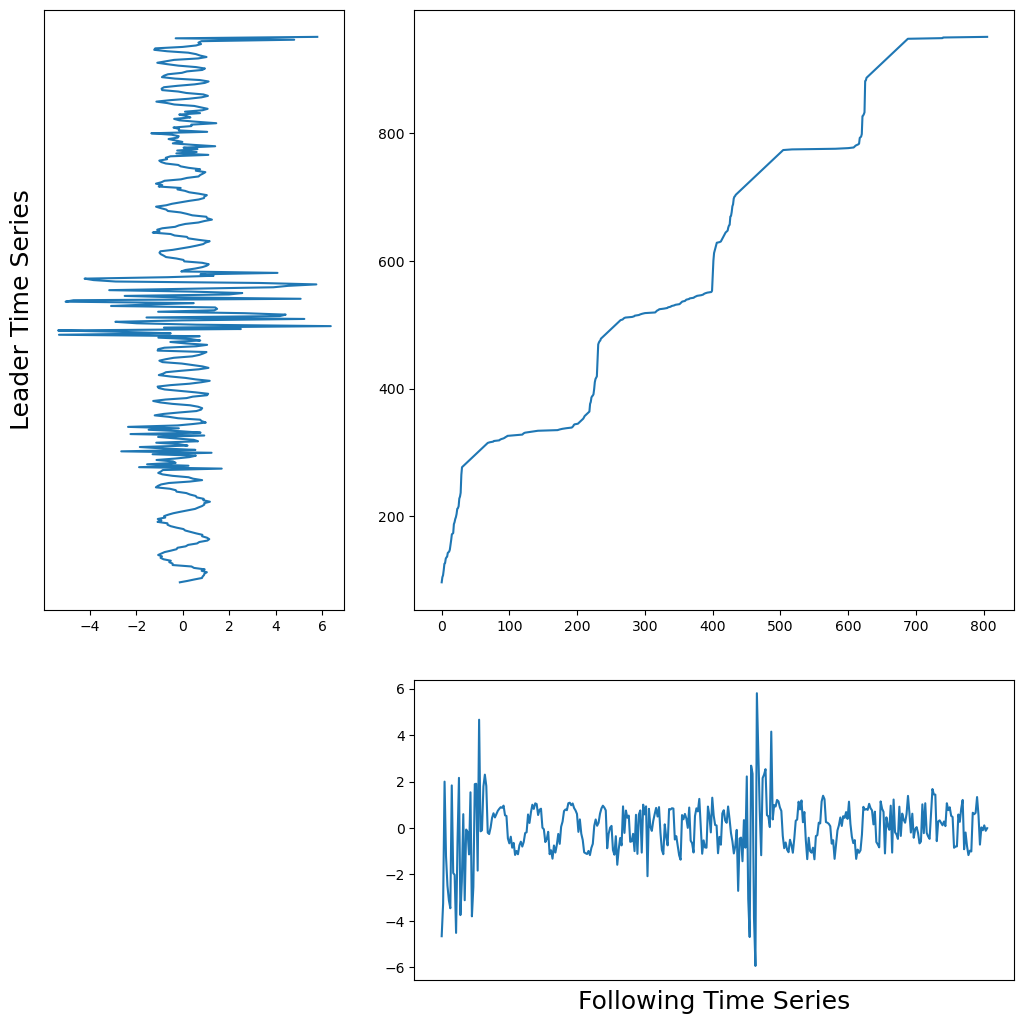}
    \caption{Following Relation Method aligns following ordered points on leader and follower Time Series of Dataset2 randomseed 883}
    \label{fig:FRdataset2}
\end{figure}


\begin{figure}
    \centering
    \includegraphics[width=0.9\linewidth]{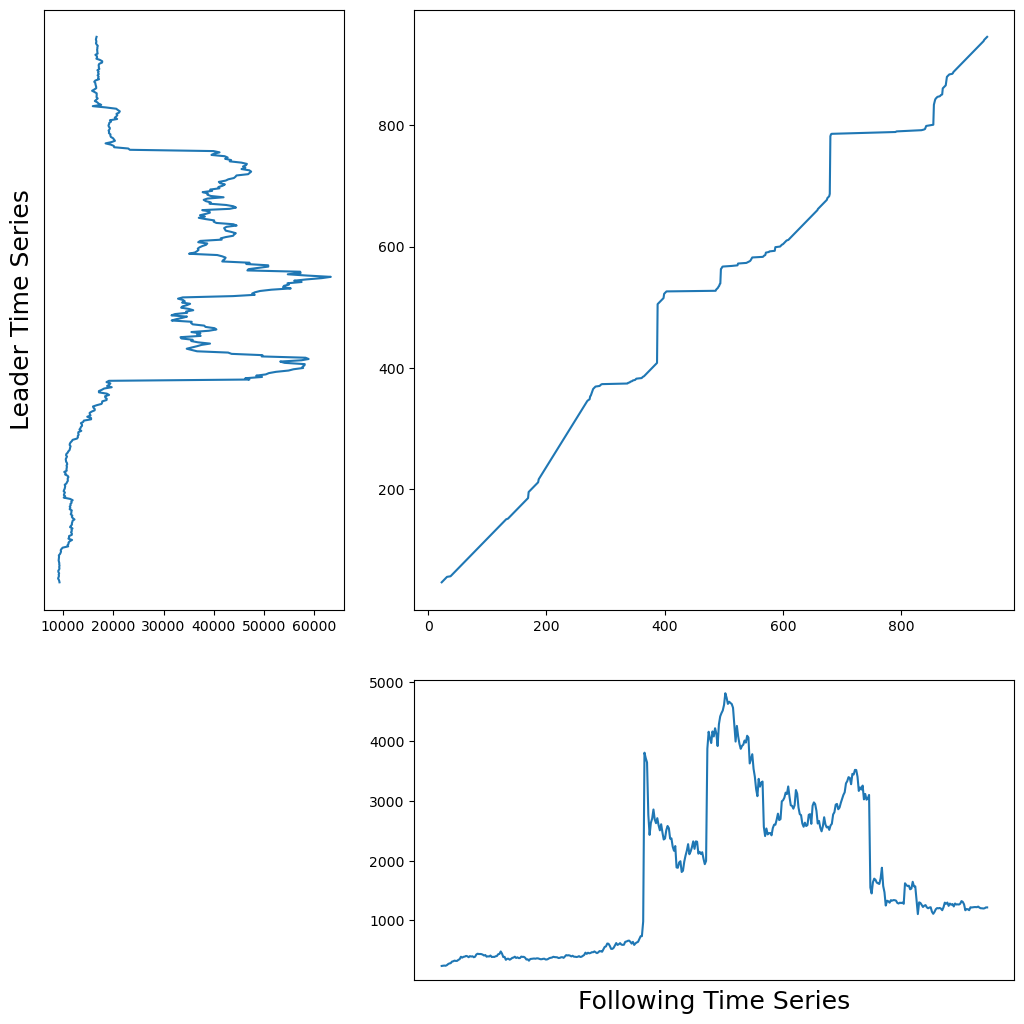}
   \caption{Following Relation Method aligns following ordered points on leader and follower Time Series of Bitcoin and Ether}
    \label{fig:FRbitcoin}
\end{figure}


\section{Case Study: Resolution of Time Series in Lost Star Singing}

Resolution of time series effects the similarity of shape. Thus, some certain degree of resolution is the suitable degree to find similarity of the following motif relation, because the following lags are varied and cause in mismatching of the similar part on leader and follower time series.  

If the Time Series leads and follows in some feature of the data as only-lyrics Lost Star singing data set, which the following relation exists in the rhythm, too high resolution makes similarity of shape is low as in Table~\ref{tab:intervalInTask2}. The table~\ref{tab:intervalInTask2} shows the more noise added, the better of all metrics especially for LeadVols of 0.000 std. and 0.001 std noise. The significant changing of the LeadVols to increasing noise indicates how the resolution impacts the search of similarity. The noise should be added enough in some certain degree to have clear shape of the time series. 


\begin{figure}
    \centering
    \includegraphics[width=0.9\linewidth]{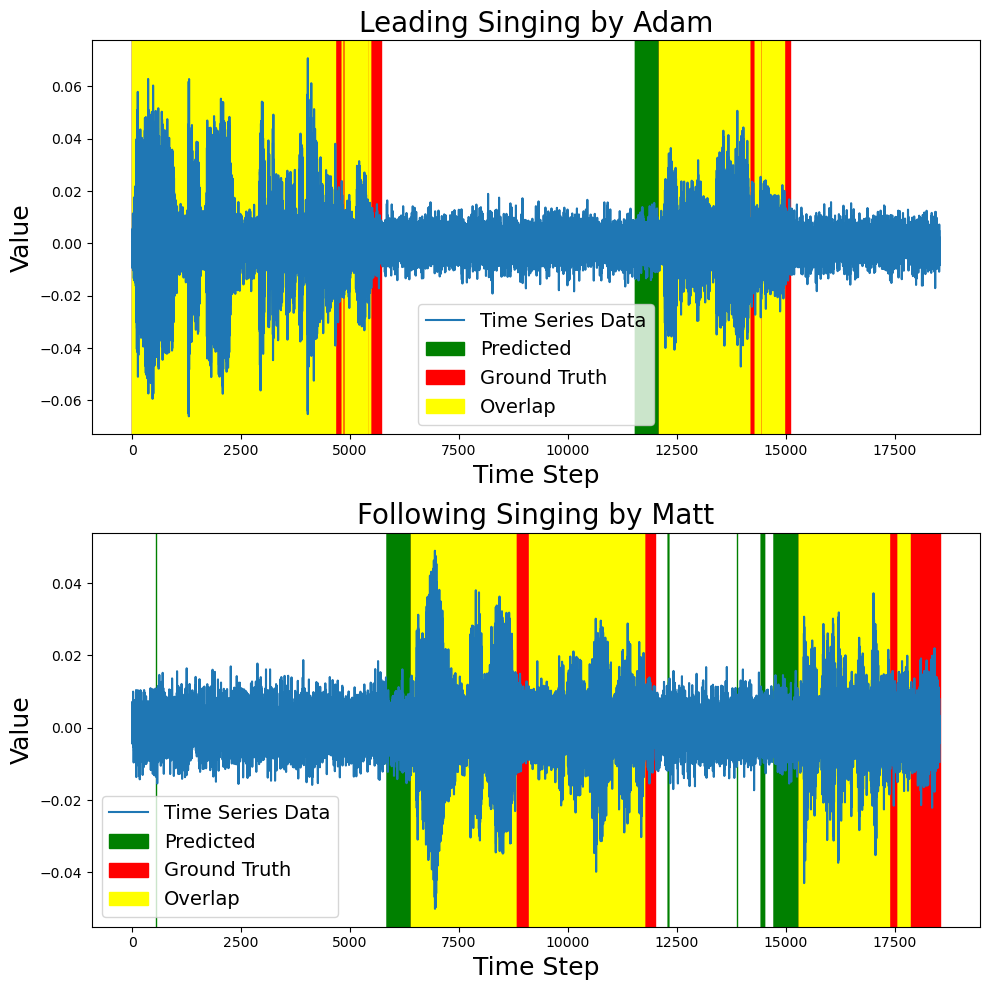}
   \caption{Following Relation Method indicates following ordered points on leader and follower Time Series of Lost Star Singing with Resolution (Gaussian noise Std. = 0.005); The best resolution to find similar motif for this data set.}
    \label{fig:FRloststar}
\end{figure}

\end{document}